\documentclass[conference,10pt]{IEEEtran}

\usepackage{graphicx}
\usepackage{subfigure}
\usepackage{hyperref}
\usepackage{url}
\usepackage{colortbl}
\usepackage[table]{xcolor}
\usepackage{bbm}
\usepackage{xcolor}
\usepackage{amsmath}
\usepackage{amsthm}
\usepackage{amssymb}
\usepackage{makecell}
\usepackage{multirow}
\usepackage{algorithm}
\usepackage{algorithmicx}
\usepackage{algpseudocode}
\usepackage{booktabs}
\usepackage{cases}

\newtheorem{proposition}{\textbf{Proposition}}

\newcommand{\etal}{\textit{et al. }}


\begin{document}
\IEEEoverridecommandlockouts
\title{FedLP: Layer-wise Pruning Mechanism for Communication-Computation Efficient Federated Learning\thanks{
        * Pingyi Fan is the corresponding author. This work was supported by the National Key Research and Development Program of China (Grant NO.2021YFA1000500(4)).}
    \thanks{\IEEEauthorrefmark{4} K. B. Letaief is supported in part by the Hong Kong Research Grant Council under Grant No. 16208921.}
    \thanks{Accepted as a conference paper by IEEE International Conference on Communications (ICC) 2023.}}




\author{\IEEEauthorblockN{Zheqi Zhu\IEEEauthorrefmark{2}, Yuchen Shi\IEEEauthorrefmark{2}, Jiajun Luo\IEEEauthorrefmark{2}, Fei Wang\IEEEauthorrefmark{3}, Chenghui Peng\IEEEauthorrefmark{3}, Pingyi Fan\IEEEauthorrefmark{2}*, and Khaled B. Letaief\IEEEauthorrefmark{4}}
    \IEEEauthorblockA{\IEEEauthorrefmark{2}Department of Electronic Engineering, Tsinghua University.\\
        Emails: \{zhuzq18, shiyc21, luo-jj18\}@mails.tsinghua.edu.cn, fpy@tsinghua.edu.cn\\
        \IEEEauthorrefmark{3}Huawei Wireless Technology Lab. Emails: \{wangfei76, pengchenghui\}@huawei.com\\
        \IEEEauthorrefmark{4}Department of ECE, Hong Kong University of Science and Technology. Email: eekhaled@ece.ust.hk
    }}


\maketitle
\IEEEpeerreviewmaketitle

\begin{abstract}
    Federated learning (FL) has prevailed as an efficient and privacy-preserved scheme for distributed learning. In this work, we mainly focus on the optimization of computation and communication in FL from a view of pruning. By adopting layer-wise pruning in local training and federated updating, we formulate an explicit FL pruning framework, FedLP (Federated Layer-wise Pruning), which is model-agnostic and universal for different types of deep learning models. Two specific schemes of FedLP are designed for scenarios with homogeneous local models and heterogeneous ones. Both theoretical and experimental evaluations are developed to verify that FedLP relieves the system bottlenecks of communication and computation with marginal performance decay. To the best of our knowledge, FedLP is the first framework that formally introduces the layer-wise pruning into FL. Within the scope of federated learning, more variants and combinations can be further designed based on FedLP.
\end{abstract}


\section{Introduction}
\subsection{Backgrounds}
By locally training the distributed models and periodically updating the global model, federated learning (FL), first conceptualized in \cite{mcmahan2017communication}, provides an explicit paradigm for cooperative learning without sharing the privacy data. Instead of transmitting the data or intermediate outputs in the networks, only the model parameters are interacted, which significantly improves the communication efficiency.

With the continuous growth of the communication systems and the intelligent devices, it is possible to adopt FL schemes in numerous promising applications such as mobile edge computing (MEC), artificial intelligence of things (AIoT), and autonomous driving \cite{niknam2020federated}. Since leveraging AI in networks is envisioned as a core characteristic of 6G systems, FL has shown its powerful potentials on combinations with deep learning models \cite{letaief2019roadmap, yang2021federated}. On the one hand, FL naturally fits the structure of multi-user networks with distributed data and can be easily deployed for machine learning tasks \cite{wan2023global}. On the other hand, FL schemes are able to achieve the intelligent collaboration for multiagent systems \cite{zhu2021federated}.

Though FL has received rapid developments in terms of methods, models and applications, such distributed learning scheme still suffers from challenges in several aspects. This work addresses two key issues, the heterogeneity and the communication-computation efficiency. Firstly, as summarized in \cite{kairouz2021advances}, the heterogeneity of FL systems mainly comes from local data and the client devices. The heterogeneous local data, also referred to as non-iid data, occurs commonly in real-world distributed scenarios and usually cause the degradation in terms of model performance, convergence and stability. Meanwhile, due to the diversity in computation platforms, communication capabilities and the battery level of the devices, clients may meet different constraints in model scales and local processing. For example, some weak clients are not able to support sufficient local training, which leads to bad model performance and time delay for synchronized aggregation. In contrast, some strong clients may not fully utilize its devices' capability to get better service quality, leading to unfairness of the whole system. Thus, how to design heterogeneous models suitable for various clients is still an open problem.

Besides, communication is also a critical bottleneck for FL networks, especially those with massive number of clients. Since communication and computation are tightly coupled in FL systems, the interplay of these progresses impacts the model quality as well as the system efficiency. Therefore, in recent studies of FL, more attention is paid to optimizing FL schemes through communication compression and computing reduction \cite{lim2020federated}.


\subsection{Motivations \& Related Works}
As model pruning has been verified to be an efficient approach to reduce the model scales with the cost of marginal loss in accuracy, related technique has also been employed in the context of FL. In this work, we mainly focus on the pruning mechanism in FL to relieve the communication and computation restrictions.

In the literature, the combination of model compression and FL has already been investigated. Aiming to reduce the client resource requirements in FL systems, Caldas \etal proposed Federated Dropout in \cite{caldas2018expanding}, where the global model is compressed into sub-models for communication. Such schemes were extended in \cite{yu2021adaptive}, where the authors proposed the dataset-aware dynamic pruning approach to accelerate the inference on edge devices. Jiang \etal in \cite{jiang2022model} and Kumar \etal in \cite{kumar2022neuron} formulated an adaptive pruning strategies based on gradient information and neuron importance, respectively. An efficient private update scheme of federated sub-model learning was also discussed in \cite{vithana2022efficient}. Further, authors of \cite{liu2021adaptive,wen2022federated} theoretically concerned about the pruning configurations and the communication resource allocation in wireless FL.

However, negative findings in a critical paper \cite{cheng2022does} argued that dropout-based FL schemes may perform worse than simple ensemble methods. Actually, such pruning schemes in FL are borrowed from dropout in single machine learning. Locally training sub-models and aggregating them to the corresponding parts of the global model lacks explainability. The authors in \cite{horvath2021fjord} also pointed out that the order of the parameters cannot be neglected in pruned aggregation. Besides, most existing dropout-based schemes in FL employed the intra-layer pruning, which results in increment of system complexity, e.g., different operations for different functional layers. Chen \etal in \cite{chen2022fedobd} introduced the block dropout for large-scale neural network training. Inspired by the layer-wise aggregation in \cite{lee2021layer}, we first propose to explore layer-wise pruning mechanism to relieve above quagmires in this work.

\subsection{Contributions \& Paper Organization}
The main contributions of this work can be summarized as follows:
\begin{itemize}
    \item[$\bullet$] We first put forward a universal FL pruning framework, FedLP\footnotemark[1], employing the layer-wise pruning mechanism. FedLP can relieve restrictions of communication as well as computation in FL systems, and also potentially prevents model attacks in some degrees.
    \item[$\bullet$] We sketch two basic pruning schemes and the theoretical principle of FedLP for both homogeneous and heterogeneous cases. In particular, the heterogeneous scheme fits the scenarios where the clients vary from device types and computation capabilities, and thus, the local models shall be set adaptively.
    \item[$\bullet$] We develop experiments to evaluate the communication-computation efficiency of FedLP\footnote[1]{The codes in this work are available at \url{https://github.com/Zhuzzq/FedLP}}. The outcomes suggest that such layer-wise pruning mechanism significantly reduces the communication loads and computational complexity with controllable performance loss.
\end{itemize}

The remaining of this article is organized as follows. In Section \ref{section preliminaries} we introduce the preliminaries of this work and illustrate how the basic idea of layer-wise pruning is formulated through a simple experiment. In Section \ref{section algorithms}, we present two typical schemes of FedLP for homogeneous and heterogeneous scenarios. The corresponding algorithms and a theoretical principle will also be developed. The detailed experimental results and more discussions are presented in Section \ref{section evaluation}. Finally, in Section \ref{section conclusion}, we conclude this work and point out several potential research directions.

\section{Preliminaries and Layer-wise Pruning}
\label{section preliminaries}
In this section, we first briefly introduce the preliminaries of this work. Then, we illustrate the key idea of layer-wise pruning, which inspires us to sketch FedLP framework.

\subsection{Federated Learning}
Recapping a classical horizontal federated learning (HFL) system, there exists $N$ distributed clients with their own local datasets, $\{\mathcal{D}_1,\cdots,\mathcal{D}_N\}$, and local models, $\{\boldsymbol{\theta}_1,\cdots,\boldsymbol{\theta}_N\}$. For privacy preserving and communication efficiency, FL carries out procedures of local training and periodic model aggregation. A federated period processes as follows: 1) Clients train local models with local data; 2) Parameter server collects local models uploaded by $K$ clients and aggregates them as the global model $\bar{\boldsymbol{\theta}}$; 3) Clients download the updated global model for further training.

As for federated updating, at each global epoch $t$, HFL selects a set of participators with $K$ clients as $P_t$, and proceeds the parameter aggregation:
\vspace{-0.2em}
\begin{equation}
    \label{def-fl}
    \bar{\boldsymbol{\theta}}_t\leftarrow\sum\limits_{k\in P_t}\frac{\omega_k}{\sum_{m\in P_t}\omega_m}\boldsymbol{\theta}_{k,t},
\end{equation}
where $\{\omega_k\}$ is the aggregation weights and $\boldsymbol{\theta}_{k,t}$ is the local model of client $k$ after the local training in $t$-th global epoch. In particular, it reduces to the most popular scheme, FedAvg, when the weights are set as $\omega_k=\frac{|\mathcal{D}_k|}{\sum|\mathcal{D}_m|}$.

\subsection{A Simple Test and Layer-wise Pruning}
\label{sec simple test}
To further optimize the communication and computation progresses, pruning is a simple but efficient method. However, as mentioned above, existing pruning methods in FL are migrated from the traditional machine learning fields. Namely, the connections between neurons in MLP (multi-layer perceptron) or the filters in CNN (convolutional neural network) are dropped in order to reduce the model scales for training locally or the parameter quantities for transmitting. As shown in the left of Fig. \ref{hv pruning}, we rethink such pruning mechanism as a vertical cut since some inner-layer neurons are detached and the intermediate features of middle layers might be down-scaled. Then, it is natural to consider horizontal scheme, which conducts a layer-wise cut. In single machine learning, detaching a whole layer leads to a completely different model. Thus, such schemes are not termed as a pruning technique. Nevertheless, within the context of FL, clients shall also horizontally cut their model and contribute to the global model together, as shown in the right of Fig. \ref{hv pruning}. Briefly, vertical cut keeps all layers and drops some neurons in each layer, while horizontal cut drops some layers and keeps all neurons for preserved layers.
\begin{figure}[htbp]
    \centering
    \subfigure[Two types of pruning mechanisms.]{
        \label{hv pruning} 
        \begin{minipage}[t]{0.27\textwidth}
            \includegraphics[width=1\linewidth]{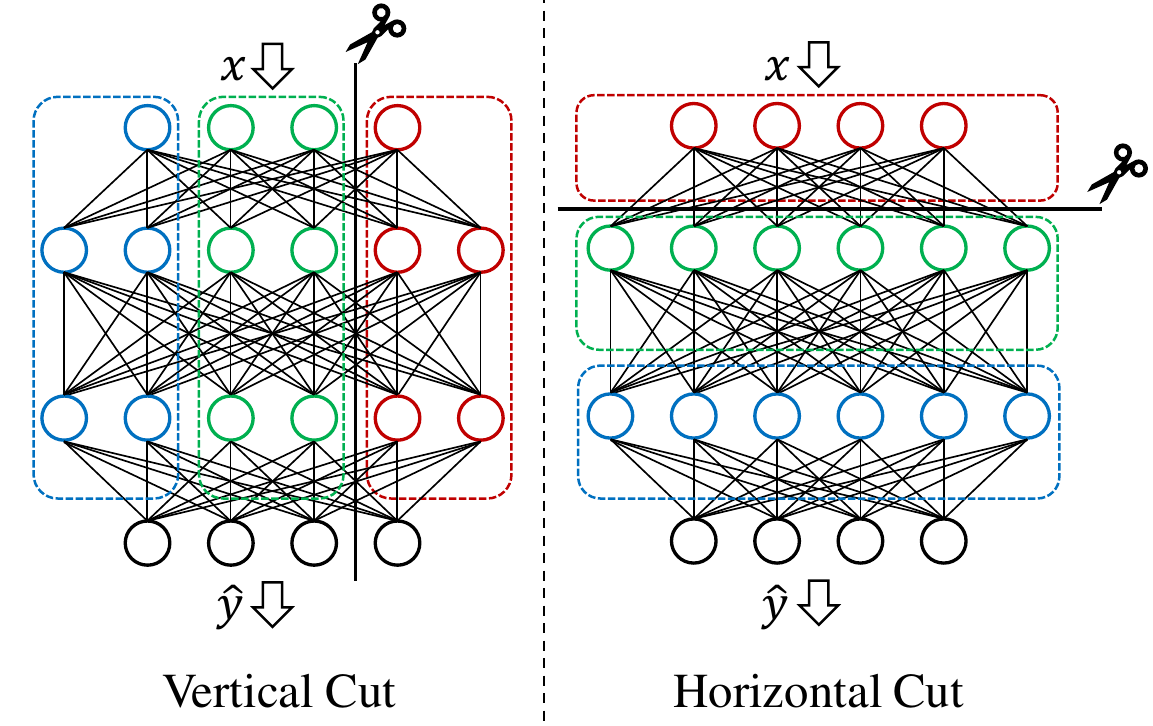}
        \end{minipage}}%
    \subfigure[Global model performance.]{
        \label{hv acc} 
        \begin{minipage}[t]{0.2\textwidth}
            \includegraphics[width=1\linewidth]{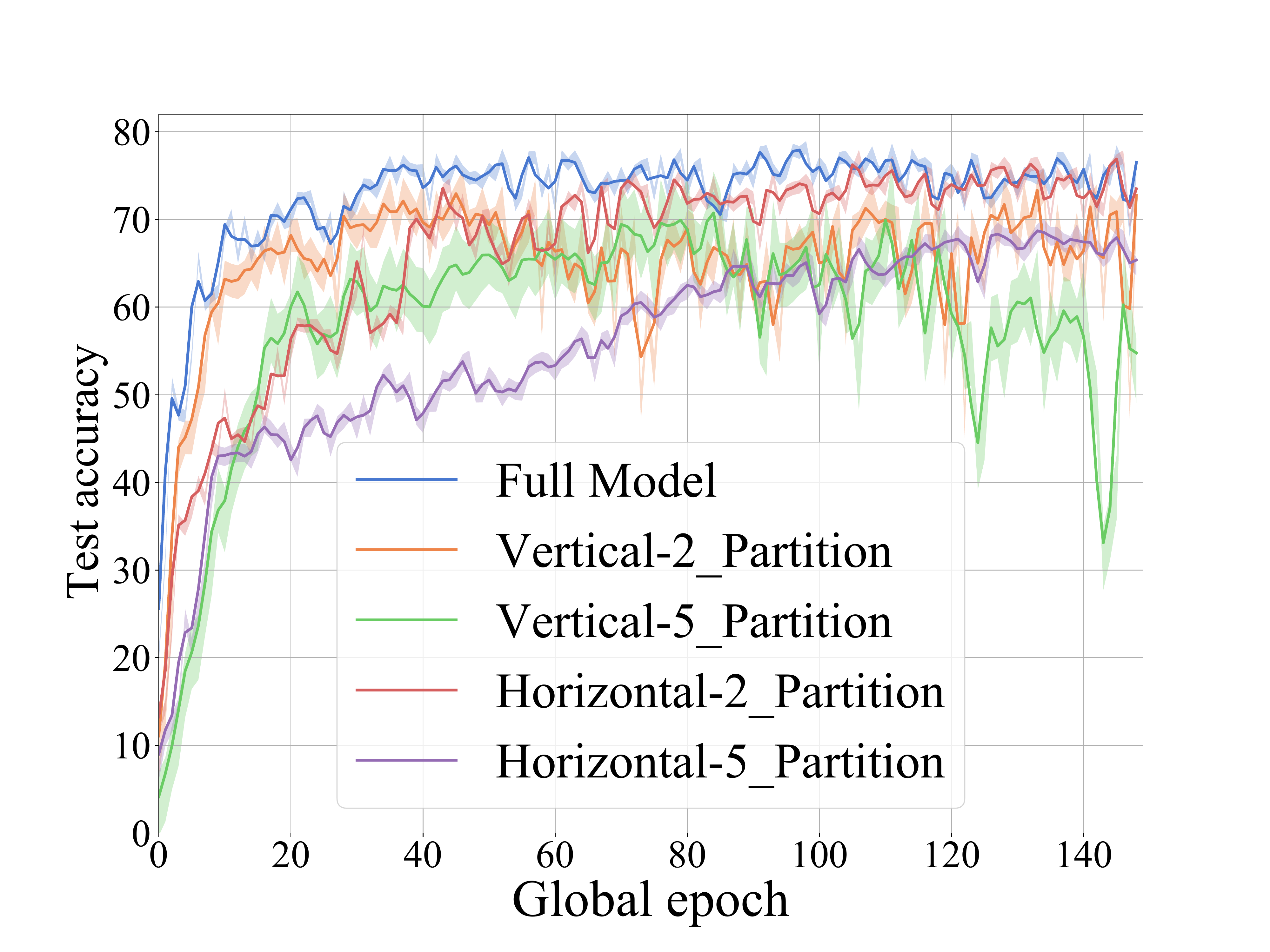}
        \end{minipage}}
    \vspace{-0.2cm}
    \caption{A simple comparison for vertical/horizontal pruning in FL.}
    \label{simple test fig} 
\end{figure}

Inspired by such intuition, we develop a simple experiment to check the performance of horizontal cut. The experiment is implemented on a FL system with 100 clients, 0.5 participation rate and non-i.i.d. data divided from Fashion-MNIST dataset. For better model performance, we only consider the pruned aggregation, which means that clients train the full model locally and only upload the pruned model. For vertical cut, the parameters within every layer is uniformly divided into 2/5 partitions and each client only upload one partition for model aggregation. Same for horizontal cut, the full whole layers are splitted into 2/5 partitions and each client upload one partitions. The comparison results are presented in Fig. \ref{hv acc}. One can find that for same partition count, horizontal cut reaps higher accuracy as well as better stability. Especially, the vertical cut with 5 partition seems not to converge after quite more global epochs.

Based on the above illustration, we deem that the horizontal cut strategies in FL have advantages over the vertical cut. Therefore, we further extend such schemes to layer-wise pruning mechanism and formulate the compressed FL framework, named as FedLP. Apart from the possible outperform beyond traditional pruning mechanism, layer-wise pruning is easier to be applied to different models. For example, clients do not have to figure out whether the model consists of fully-connected (FC) layers or convolutional (Conv) layers before pruning. This is because the whole layer is treated as the smallest pruning unit and the inner-layer structure can be neglected. Thus, we also regard the layer-wise pruning as a model-agnostic mechanism.

\section{FedLP: Frameworks and Algorithms}
\label{section algorithms}
In this section, we will formally propose the basic framework of FedLP. Then, two specific schemes and their corresponding algorithms for homogeneous and heterogeneous models will be designed respectively.

\begin{figure}[htbp]
    \centering
    \subfigure[Homogeneity scheme.]{
        \label{fig: homo scheme} 
        \begin{minipage}[t]{0.23\textwidth}
            \includegraphics[width=1\linewidth]{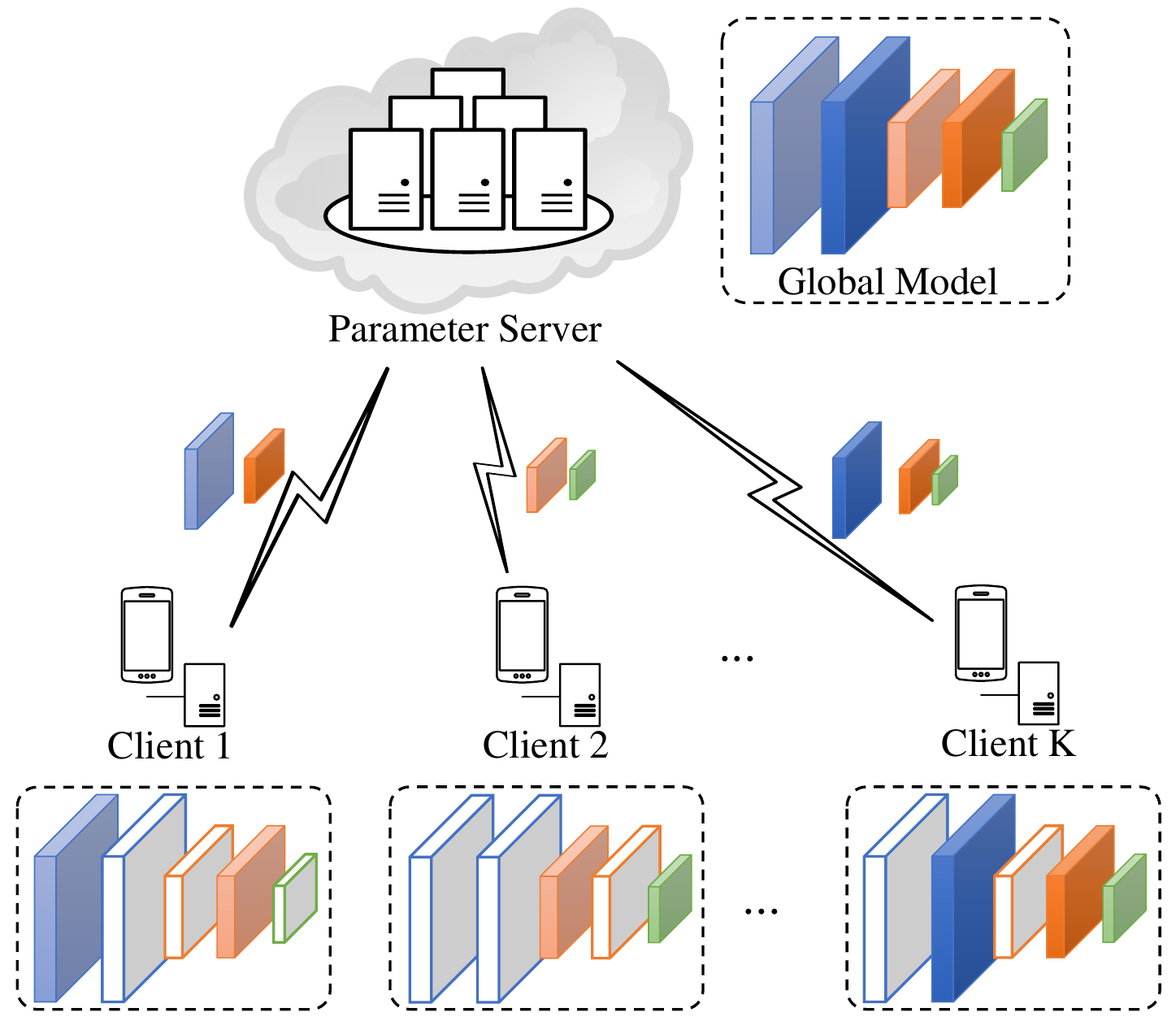}
        \end{minipage}}%
    \subfigure[Heterogeneity scheme.]{
        \label{fig: hetero scheme} 
        \begin{minipage}[t]{0.23\textwidth}
            \includegraphics[width=1\linewidth]{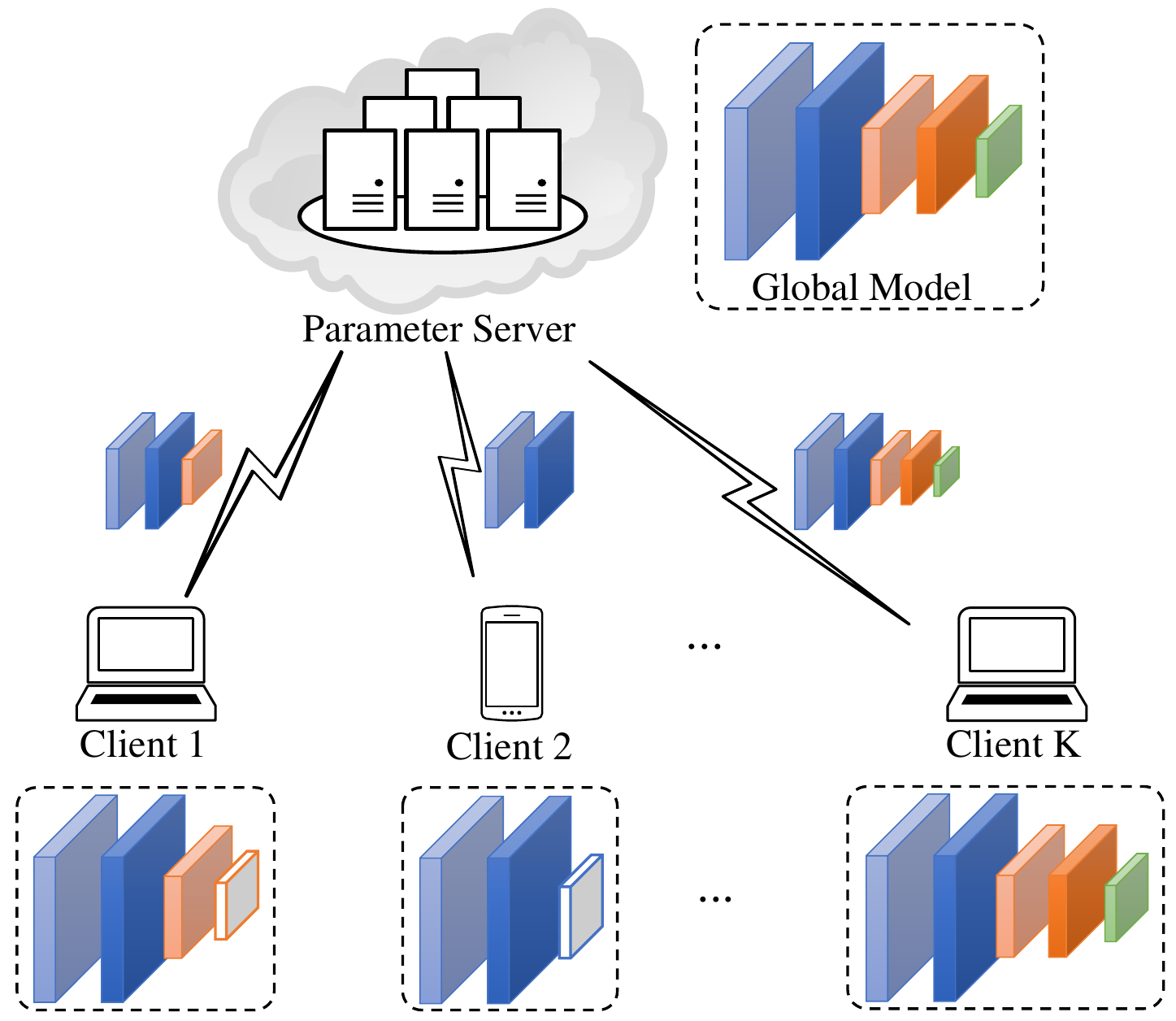}
        \end{minipage}}
    \vspace{-0.1cm}
    \caption{Two typical FedLP schemes based on different settings of local model. The inactive layers (in gray) are removed from the sub-models for uploading.}
    \vspace{-0.2cm}
    \label{fig: framework} 
\end{figure}

The key idea of FedLP is that the layer-wise pruning mechanism is adopted in the phases of local training and model aggregation. After finishing the local training of every global epoch: 1) Each client only uploads some layers to parameter server; 2) The parameter server aggregates the pruned sub-models and obtains the updated global model; 3) Clients download either the corresponding sub-models or full models to update local models. More specifically, we rewrite a full model $\boldsymbol{\theta}$ with $L$ layers as: $\boldsymbol{\theta}:=\left[\boldsymbol{\theta}^1,\boldsymbol{\theta}^2,\cdots,\boldsymbol{\theta}^L\right]$,
where $\boldsymbol{\theta}^l$ represents the parameters of $l$-th layer. Assume that each client $k$ upload layers of index in $\mathcal{L}_k$ after pruning. The pruned model of each client can be represented by:
\vspace{-0.4em}
\begin{equation}
    \label{pruned model}
    \tilde{\boldsymbol{\theta}}_k=\left[\boldsymbol{\theta}^l_k\right],\ \ l\in\mathcal{L}_k.
\end{equation}
Under FedLP aggregation, each layer is operated independently. Let the indicator function $\mathbbm{1}^l_k\in\{0,1\}$ denotes whether layer $l$ is included in the pruned local model of client $k$. With the pruning information, the layer-wise aggregation rule shall be modified as:
\vspace{-0.2em}
\begin{equation}
    \label{def-fedlp}
    \bar{\boldsymbol{\theta}}^l_t\leftarrow\sum\limits_{k\in P_t}\frac{\mathbbm{1}^l_k\cdot\omega_k}{\sum_{m\in P_t}\mathbbm{1}^l_m\cdot\omega_m}\boldsymbol{\theta}^l_{k,t}.
\end{equation}


Furthermore, we naturally consider the details of the implementation with the homogeneous clients and the heterogeneous clients. The corresponding two FedLP schemes are displayed in Fig. \ref{fig: framework}.

\subsection{Homogeneity scenario}
All clients possess the full global model, which is also a basic assumption in traditional FL. As shown in Fig. \ref{fig: homo scheme}, each client trains its local model according to local settings. Before uploading the parameters, clients carry out a layer-wise pruning to form the models for aggregation. As presented in Algorithm \ref{alg-fedlp-homo}, we formulate a probabilistic rule for homogeneity FedLP, where layer $l$ of client $k$ contributes to aggregation with probability $p^l_k$, termed as layer-preserving-rate (LPR). Before uploading, each client forms the layer-wise pruned model through the layer preserving indicators $\big\{\zeta^l_k\big\}$, which is a binary variable following $Bernoulli(p^l_k)$, i.e.,
\begin{equation}
    \label{eq-zeta}
    \zeta^l_k=\left\{
    \begin{aligned}
         & 1 & \ \ {\rm with\ probability} & \ p^l_k,   \\
         & 0 & \ \ {\rm with\ probability} & \ 1-p^l_k.
    \end{aligned}
    \right.
\end{equation}
With all pruned local models, parameter server carries out the layer-wise pruned aggregation for each layer according to \eqref{def-fedlp} and obtains the updated global model $\bar{\boldsymbol{\theta}}_t$. While parameter download, clients receive the full global model and update their local models. Besides, guaranteed by Law of large numbers, all layers contribute to the aggregation with the proportion $\{p^l_k\}$.

\vspace{0.01em}
\begin{algorithm}[htbp]
    \caption{FedLP For homogeneous model setting.}
    \label{alg-fedlp-homo}
    {\bf Initialization:} local models $\big\{\boldsymbol{\theta}_{k,0}\big\}$, LPRs $\big\{p^l_k\big\}$, system configures, etc.
    \begin{algorithmic}[1]
        \For{$t\leftarrow 1\ \mathbf{to}\ {\rm max\_epoch}$}
        \State Select $K$ clients as participator set, $P_t$;
        \For{client $k$ in $P_t$, parallelly} \Comment{\textcolor{cyan}{participator side}}
        \State Update $\boldsymbol{\theta}_{k,t}\leftarrow{Local\_Train}(\boldsymbol{\theta}_{k,t-1};\mathcal{D}_k)$;
        \State Generate the indicator variables: $\big\{\zeta^l_k\big\}\sim\big\{p^l_k\big\}$;
        \State Form the pruned local model $\tilde{\boldsymbol{\theta}}_{k,t}$ through $\big\{\zeta^l_k\big\}$;
        \State Upload $\tilde{\boldsymbol{\theta}}_{k,t}$ to parameter server;
        \EndFor
        \State Aggregate each layer $\bar{\boldsymbol{\theta}}^l_t$ by \eqref{def-fedlp};\Comment{\textcolor{orange}{server side}}
        \State Update each local model: $\boldsymbol{\theta}_{k,t}\leftarrow\bar{\boldsymbol{\theta}}_t$;\Comment{\textcolor{teal}{client side}}
        \EndFor
    \end{algorithmic}
    {\bf Output:} global model: $\bar{\boldsymbol{\theta}}_t$.
\end{algorithm}

\subsection{Heterogeneity scenario}
Clients train local models of different scales, which is more practical in real-world applications and extends the original FL. For example, in some scenarios, clients may include the mobile devices which has more restrictions on the computation capability, communication bandwidth, energy consumption, etc. Thus, both the local training and the uploading traffic shall be adaptively optimized. Model heterogeneity is always acknowledged as a major challenge in FL because it is hard to design aggregation strategies for various local models. Nevertheless, by adopting the layer-wise pruning mechanism, we can easily build the FedLP schemes for heterogeneous cases. As shown in Fig. \ref{fig: hetero scheme}, clients train sub-models with part of layers and upload them for aggregation. We measure the model complexity of client $k$ using the layer count (LC), $L_k$, which can be determined by the device capability. The local model assigned to client $k$ consists of first $L_k$ layers. To match the data dimensions, the clients with pruned models personalize $\boldsymbol{\theta}^O_{k}$ as their last output layers (in gray).
The pruning assignment rule can be formulated as:
\vspace{-0.4em}
\begin{equation}
    \label{eq-hetero assign}
    \tilde{\boldsymbol{\theta}}_k=\left\{
    \begin{aligned}
         & \Big[\boldsymbol{\theta}^1_{k},\cdots,\boldsymbol{\theta}^{L_k}_{k},\boldsymbol{\theta}^O_{k}\Big] & \ \ {\rm if}\ L_k<L, \\
         & \Big[\boldsymbol{\theta}^1_{k},\cdots,\boldsymbol{\theta}^{L}_{k}\Big]                             & \ \ {\rm if}\ L_k=L.
    \end{aligned}
    \right.
    \vspace{-0.4em}
\end{equation}
Then we develop the FedLP algorithm for heterogeneous scenarios as Algorithm \ref{alg-fedlp-hetero}. Note that the personalized layers, $\Big\{\boldsymbol{\theta}^O_{k}\Big\}$, if exists, are only trained locally and will neither be uploaded for aggregation nor updated by downloaded model.
\addtolength{\topmargin}{0.01in}
\begin{algorithm}[htbp]
    \caption{FedLP For heterogeneous model setting.}
    \label{alg-fedlp-hetero}
    {\bf Initialization:} model LCs $\big\{L_k\big\}$, system configures, etc.
    \begin{algorithmic}[1]
        \For{each client $k\leftarrow 1\ \mathbf{to}\ N$}\Comment{\textcolor{teal}{pruning initialization}}
        \State Assign local model $\tilde{\boldsymbol{\theta}}_{k,0}\leftarrow\left[\boldsymbol{\theta}^1_{k,0},\cdots,\boldsymbol{\theta}^{L_k}_{k,0},\big(\boldsymbol{\theta}^O_{k,0}\big)\right]$;
        \EndFor
        \For{$t\leftarrow 1\ \mathbf{to}\ {\rm max\_epoch}$}
        \State Select $K$ clients as participator set, $P_t$;
        \For{client $k$ in $P_t$, parallelly}\Comment{\textcolor{cyan}{participator side}}
        \State Update $\tilde{\boldsymbol{\theta}}_{k,t}\leftarrow{Local\_Train}(\tilde{\boldsymbol{\theta}}_{k,t-1};\mathcal{D}_k)$;
        \State Upload $\tilde{\boldsymbol{\theta}}_{k,t}[1:L_k]$ to parameter server;
        \EndFor
        \State Aggregate each layer $\bar{\boldsymbol{\theta}}^l_t$ by \eqref{def-fedlp};\Comment{\textcolor{orange}{server side}}
        \State Update each local model: $\boldsymbol{\theta}_{k,t}\leftarrow\bar{\boldsymbol{\theta}}_t$;\Comment{\textcolor{teal}{client side}}
        \EndFor
    \end{algorithmic}
    {\bf Output:} global model: $\bar{\boldsymbol{\theta}}_t$.
\end{algorithm}

\subsection{Theoretical Result}
Due to the page limitation, the detailed theoretical analysis on convergence are skipped and will be presented in our future works. Here, we only give a proposition to show the impact on global gradient caused by FedLP.
\begin{proposition}
    \label{prop}
    Assume that local training is independent with pruning operations. In fairness case where $\omega_k\equiv 1/K$ and $p^l_k\equiv p$ for all $k=1,\cdots,K$, FedLP gets $(1-p)^K$ convergence rate decay compared to non-pruned FL.
\end{proposition}
\begin{proof}
    \label{prop-proof}
    Let $\{g^l_k\}$ denote the accumulated local gradients of $\{\boldsymbol{\theta}^l_k\}$ after local training. We use $\zeta^l_k$ in \eqref{eq-zeta} to represent the participation indicator of $\boldsymbol{\theta}^l_k$. By \eqref{def-fedlp}, the global aggregated gradient of $l$-th layer is $\hat{g}^l=\sum_{k}\frac{\zeta^l_k}{\sum_{m}\zeta^l_m}g^l_k$. The aggregated gradient of non-pruned FL is $\bar{g}^l=\frac{1}{K}\sum_{k}g^l_k$.
    Then the expectation of the pruned gradients can be obtained by:
    \vspace{-0.2em}
    \begin{align}
        \label{prf-eq1}
        \mathbb{E}\hat{g}^l & =\mathbb{E}\left\{\sum\limits_{k=1}^K\frac{\zeta^l_k}{\sum_{m}\zeta^l_m}g^l_k\right\}=\sum\limits_{k=1}^K\mathbb{E}\left\{\frac{\zeta^l_k}{\sum_{m}\zeta^l_m}g^l_k\right\} \\
                            & =\sum\limits_{k=1}^K\mathbb{E}\left\{\mathbb{E}\Big\{\frac{\zeta^l_k}{\sum_{m}\zeta^l_m}g^l_k\Big|\zeta^l_k\Big\}\right\}                                                  \\
                            & =\sum\limits_{k=1}^K p\cdot\mathbb{E}\left\{\frac{1}{1+\sum_{m\neq k}\zeta^l_m}\cdot g^l_k\right\}                                                                         \\
                            & =\sum\limits_{k=1}^K p\left[\sum\limits_{m=1}^K\frac{1}{m}\binom{K-1}{m-1}p^{m-1}(1-p)^{K-m}\right]\mathbb{E}g^l_k                                                         \\
                            & \overset{\textcircled{\small{1}}}{=}\sum\limits_{k=1}^K\left[\sum\limits_{m=1}^K\frac{1}{K}\binom{K}{m}p^m(1-p)^{K-m}\right]\mathbb{E}g^l_k                                \\
                            & =\left[1-(1-p)^K\right]\sum\limits_{k=1}^K \frac{\mathbb{E}g^l_k}{K}=\left[1-(1-p)^K\right]\mathbb{E}\bar{g}^l \label{prf-eq-s}
    \end{align}
    where \textcircled{\small{1}} holds because $\frac{1}{m}\binom{K-1}{m-1}=\frac{(K-1)!}{(K-m)!m!}=\frac{1}{K}\binom{K}{m}$. \eqref{prf-eq-s} means that the aggregated gradient scale of FedLP decreases compared to the non-pruned one, which implies $(1-p)^K$ convergence rate decay.
\end{proof}\vspace{-0.2em}
This theoretical result is significant and shows that the impacts of FedLP on convergence can be mitigated by increasing the participation clients or the LPRs $\{p^l_k\}$.

Overall, we give a brief summary for these two FedLP schemes: Homogeneous scheme mainly focuses on the communication progress and concerns less about the local training; On the contrary, for the heterogeneous scheme, both local computation and parameter communication are reduced. It is notable that such a random pruning processing will protect the FL model from attacks in some degrees and strengthen the system robustness as well as security since the attackers are unaware of the exact layer indexes.



\section{Evaluations}
\label{section evaluation}
In this section, we carry out several experiments to evaluate the performance of FedLP under two mentioned scenarios.
\subsection{Experiment Setups}
We develop the experiments in an image classification FL task under CIFAR-10 dataset. For basic configuration, we build up a FL system with 100 clients in total. The participation rate is set as 0.1, which means that 10 clients are randomly selected for aggregation in every global round. Before aggregation, clients proceed 5 epochs to train the local models.

\subsubsection{Global Model}
A CNN based model with 6 Conv layers is adopted as the global model. Batch normalization (BN) and maxpooling operations are also conducted following each Conv layer. At the end, a FC layer is placed to assemble the features, followed by another FC as the output layer.
\subsubsection{Data Split}
We conduct the experiments under three popular data settings of FL, iid, mixed non-iid and Dirichlet non-iid. For iid split, the training samples are randomly assigned to 100 clients, which means that each client possesses 500 images of uniform categories. Under mixed non-iid (M-niid) split, the training samples are sorted into shards and partitioned to clients. We set the size of each shard as 250 with 5\% uniformly sampled from all categories and each client takes 2 shards. Dirichlet non-iid (D-niid) split is also a widely mentioned data partition rule \cite{hsu2019measuring}. The samples of each category are divided into $N$ parts according to a Dirichlet distribution with parameter $\alpha=1$, so that the clients own training data of different volumes.

\subsubsection{Pruning Strategies}
We adopt FedAvg as the baseline and two proposed layer-wise pruning schemes are implemented. For homogeneous cases, we employ the consistent LPRs, $p^l_k\equiv p$, abbreviated as FedLP-Homo($p$). For heterogeneous cases, we prune the global model into 5 ordered layer-sequences with different LCs. These 5 sub-models are assigned to clients according to LC distributions. Specifically, the parameter $l$ of FedLP-Hetero($l$) represents the case where sub-models with $l$ LC are mostly assigned with probability 0.6 and other sub-models takes 0.1 respectively. In particular, the parameter 'u' means that all sub-models are chosen uniformly.

\begin{table}[htbp]
    \vspace{-1.0em}
    \caption{\upshape A numerical comparison on accuracy, communication and computation costs.}
    \label{tab: compare}
    \centering
    \begin{tabular}{cccc}
        \toprule

        \multirow{2}*{Schemes} & \multirowcell{2}{Test accuracy (\%)                                                 \\ iid / D-niid / M-niid} & \multirowcell{2}{Comm. \\ \#param (k)} & \multirowcell{2}{Comp. \\MFLOPs} \\
        \\
        \midrule
        FedAvg                 & 77.94 / 77.67 / 67.57                            & 1102.93         & 36.36          \\
        \midrule[0.3pt]
        FedLP-Homo(0.1)        & 75.32 / 71.30 / 44.21                            & \textbf{606.61} & 36.36          \\
        FedLP-Homo(0.3)        & 78.20 / 74.92 / 63.24                            & \textbf{716.91} & 36.36          \\
        FedLP-Homo(0.5)        & 77.60 / 77.13 / 66.01                            & \textbf{827.20} & 36.36          \\
        FedLP-Homo(0.7)        & \textbf{78.47} / \textbf{77.71} / \textbf{70.29} & \textbf{937.49} & 36.36          \\
        \midrule[0.3pt]
        FedLP-Hetero(1)        & 66.00 / 67.66 / 37.30                            & \textbf{169.60} & \textbf{17.73} \\
        FedLP-Hetero(3)        & 68.82 / 68.29 / 39.54                            & \textbf{225.28} & \textbf{24.62} \\
        FedLP-Hetero(u)        & 72.42 / 64.65 / 57.51                            & \textbf{318.66} & \textbf{23.83} \\
        FedLP-Hetero(5)        & \textbf{76.28} / \textbf{76.34} / \textbf{65.69} & \textbf{710.80} & \textbf{30.10} \\
        \bottomrule
    \end{tabular}
    \vspace{-2.0em}
\end{table}
\begin{figure}[htbp]
    \centering
    \subfigure[Under iid data.]{
        \label{fig: iid acc} 
        \begin{minipage}[t]{0.154\textwidth}
            \centering
            \includegraphics[width=1\linewidth]{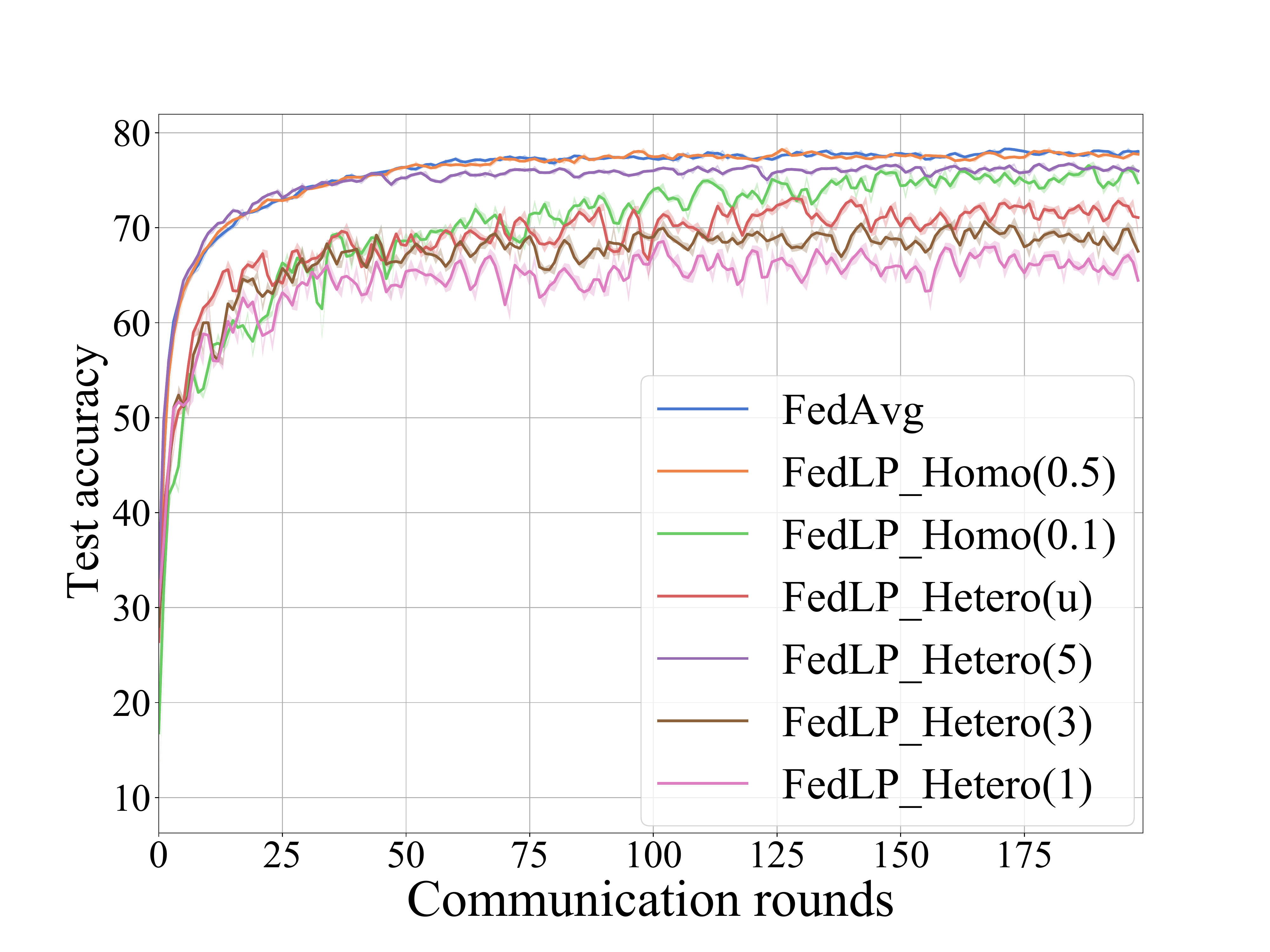}
        \end{minipage}}%
    \hspace{-0.5em}
    \subfigure[Dirichlet-niid.]{
        \label{fig: dniid acc} 
        \begin{minipage}[t]{0.154\textwidth}
            \centering
            \includegraphics[width=1\linewidth]{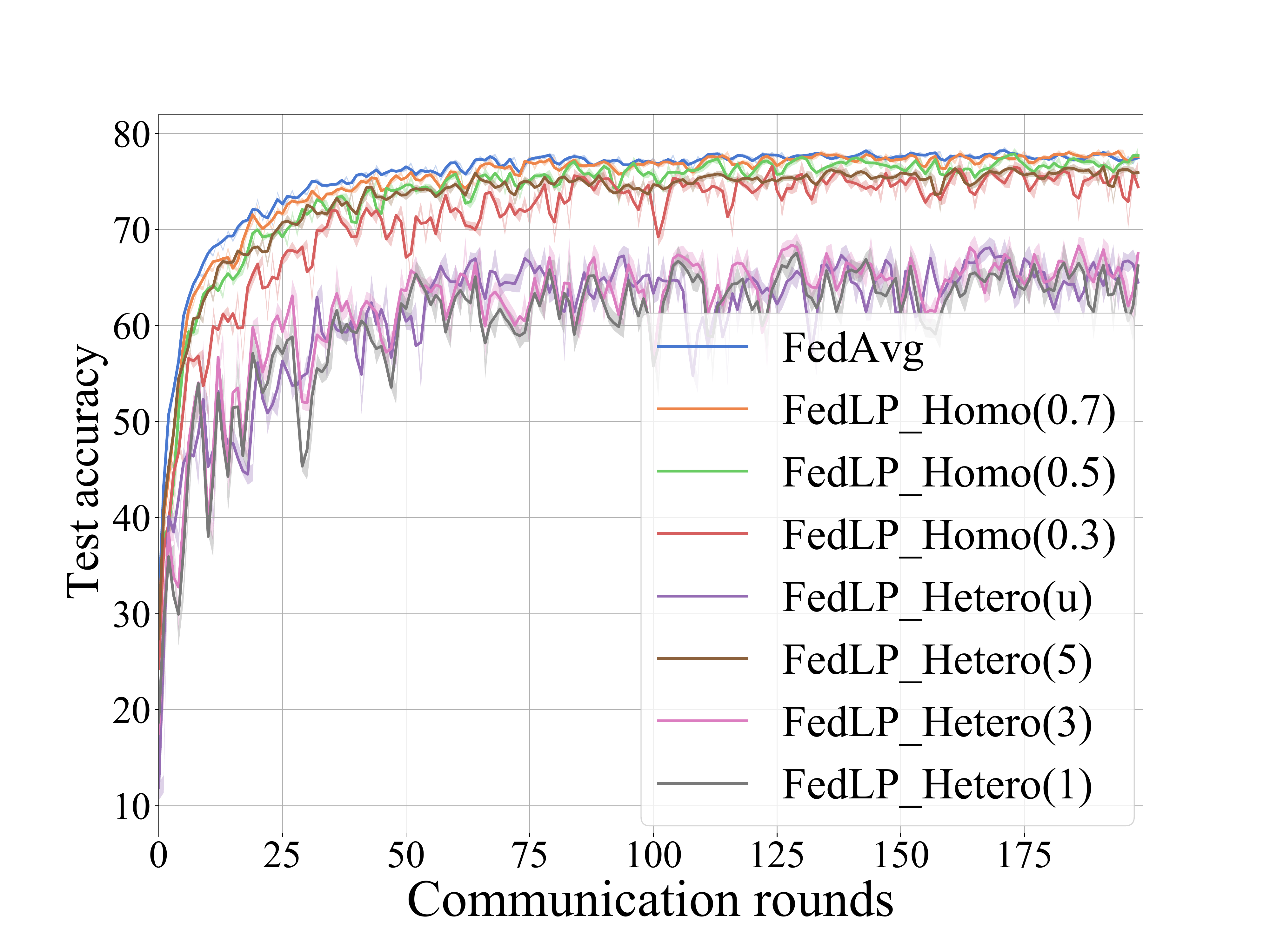}
        \end{minipage}}
    \hspace{-0.85em}
    \subfigure[Mixed-niid.]{
        \label{fig: mniid acc} 
        \begin{minipage}[t]{0.154\textwidth}
            \centering
            \includegraphics[width=1\linewidth]{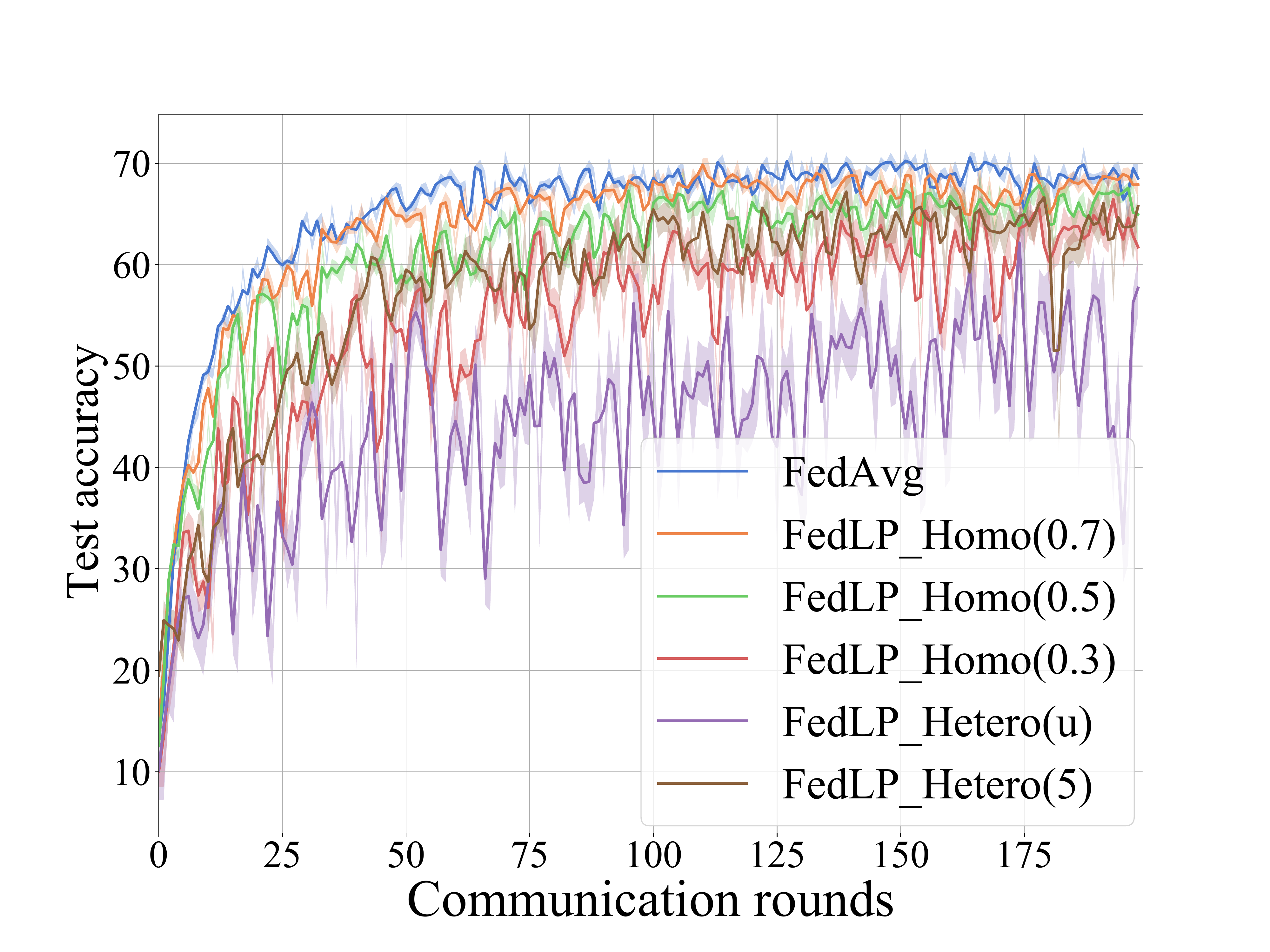}
        \end{minipage}}
    \vspace{-0.1cm}
    \caption{Comparisons under different FL data splits based on CIFAR-10.}
    \vspace{-0.1cm}
    \label{fig: acc} 
\end{figure}
\subsection{Accuracy Performance}
We evaluate the performance of the proposed layer-wise pruning schemes under three FL data settings. The test accuracy curves of global models are plotted in Fig. \ref{fig: acc} and the average numerical results are listed in Table \ref{tab: compare}. One can observe that for higher LPRs (0.7), FedLP-Homo performs even better than original FedAvg under iid and non-iid settings, but saves 30\% communication loads for model upload. Such improvement on generalization capability is also a result of the random LP processing. For intermediate LPRs, FedLP-Homo(0.5) achieves similar accuracy, convergence and stability as communication round increases. The convergence performance of large LPR also fits the Proposition \ref{prop}. With 0.3 LPR, the performance under iid data keeps the same. The accuracy drops by 3\% and 4\% under two non-iid settings, compared to non-pruned schemes. In particular, when LPR is set extremely low (0.1), the global models under iid and Dirichlet-niid data only lose 3\% and 6\% classification accuracy, which is still acceptable.
Large performance gap occurs under mixed-niid data training.

For heterogeneous cases, both the local training and the transmission models are pruned. As shown in Fig. \ref{fig: iid acc} and \ref{fig: dniid acc}, FedLP schemes still reach high accuracy under iid and Dirichlet-niid data. However, the model accuracy as well as the convergence evidently degrade under mixed-niid data. We explain these phenomena as the result of the non-iid degree and the heterogeneous local models. For severely non-iid data in Fig. \ref{fig: mniid acc}, vanilla FedAvg based approaches cannot handle the imbalanced training data and the divergent local models. Besides, FedLP-Hetero assigns models of different complexity and the personalized output layers are absent for aggregation, which also causes the unstability of the global model. Therefore, one can treat such scheme as an alternative solution for the scenarios where the clients are strictly constrained on communication-computation resources and the devices are of variant capability. To further improve the accuracy, the techniques against severely non-iid data shall be added.

\subsection{Communication-Computation Efficiency}
\begin{figure}[htbp]
    \vspace{-0.25cm}
    \centering
    \begin{minipage}[t]{0.43\textwidth}
        \includegraphics[width=1\linewidth]{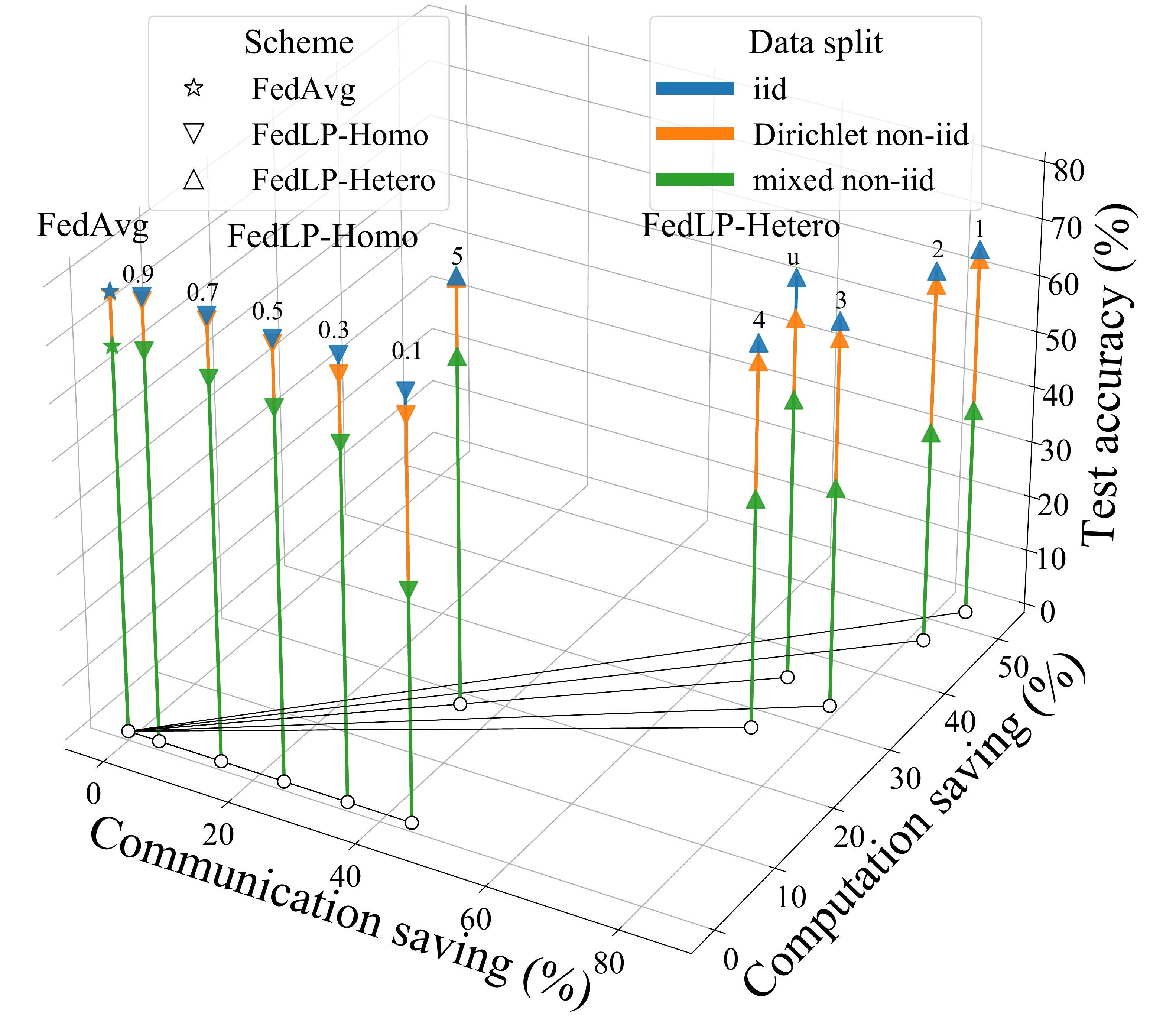}
    \end{minipage}
    \vspace{-0.25cm}
    \caption{Trade-offs: accuracy vs. communication-computation efficiency. Homo: $p=\{0.1,0.3,0.5,0.7,0.9\}$; Hetero: $\{{\rm uniform},1,2,3,4,5\}$.}
    \vspace{-0.2cm}
    \label{fig: 3d} 
\end{figure}
We also evaluate the communication-computation efficiency of FedLP. The communication loads are measured by the average parameter count in a global epoch, containing the local model upload and the global model download. The computation complexity is represented by the million floating point of operations (MFLOPs) per local model. These two measures are also listed in Table \ref{tab: compare}. Specifically, FedLP-Homos schemes reduce the data clients transmit to parameter server in different degree and execute the same local training as FedAvg. Meanwhile, under FedLP-Hetero, both local training and the model aggregation are pruned. Thus, both communication and computation costs decrease. Besides, since the traffics of uplinks and downlinks are optimized together, the excessive communication can be further eliminated. For example, the uniform FedLP-Hetero schemes requires only half the communication rates of FedLP-Homo with 0.1 LPR.

Moreover, we sketch a 3D plot of several schemes to visualize the model accuracy and the system efficiency. The x-y axis represents the percentage of communication and computation savings respectively. And the height is the test accuracy of global model. The farther a projection on x-y plane locates towards 0-point, the lower communication/computation capabilities are required, which increases the system efficiency. Fig. \ref{fig: 3d} intuitively reflects such trade-offs between model performance and system costs which provide guidance for the system designs and layer-wise pruning settings.

Above results suggest that it is not necessary either for clients to possess the full global model or for parameter server to collect all layers of local models. While the data is not highly non-iid, clients are allowed to prune some layers with acceptable accuracy decay, which significantly reduces the communication loads as well as the local computation complexity. In other words, FedLP relaxes the restrictions on communication and computation for practical FL systems.

\section{Conclusion}
\label{section conclusion}
In this work, we rethought the pruning strategies in the context of FL and proposed a layer-wise pruning mechanism, FedLP. Instead of dropping the intra-layer parameters vertically, FedLP operates pruning horizontally on each layer to improve the communication-computation efficiency of FL systems. We drew a basic sketch of layer-wise pruning by developing two probabilistic FedLP schemes for homogeneous and heterogeneous scenarios. Theoretical guarantees were also derived to interpret the convergence of FedLP. The experimental outcomes verified that FedLP reduces both communication and computation costs with the controllable loss of model performance. Moreover, FedLP is model-agnostic and can be easily deployed in different FL schemes regardless of the neural network structures and the layer types. Such an explicit pruning mechanism provides alternative ways to implement FL tasks on edge devices with variant capabilities.

Based on FedLP, more works of several aspects can be investigated in the future. Firstly, layer-wise pruning schemes such as the dynamic pruning with adaptive layer weights shall be explored to fit the changeable environments and the non-iid data. Secondly, the theoretical analysis on learning convergence and system configuration will be presented in our future works. In additional, FedLP's potentials on system robustness and model security can be further discussed.

\bibliographystyle{IEEEtran}
\bibliography{IEEEabrv, refs}

\end{document}